\documentclass[letterpaper]{article} 
\usepackage{aaai24}  
\usepackage{times}  
\usepackage{helvet}  
\usepackage{courier}  
\usepackage[hyphens]{url}  
\usepackage{graphicx} 
\usepackage{color}
\usepackage{natbib}  
\usepackage{caption} 
\usepackage{algorithm}
\usepackage{algorithmic}
\usepackage{booktabs}       

\usepackage{newfloat}
\usepackage{listings}
\usepackage{graphicx}
\usepackage{amsmath}
\usepackage{algorithm}
\usepackage{stackengine}
\usepackage{amssymb}
\usepackage{amsthm}
\usepackage{mathtools}
\usepackage{bbold}
\usepackage{scalerel}
\newtheorem{theorem}{Theorem}

\newtheorem{corollary}{Corollary}

\DeclareCaptionStyle{ruled}{labelfont=normalfont,labelsep=colon,strut=off} 
\floatstyle{ruled}
\newfloat{listing}{tb}{lst}{}
\floatname{listing}{Listing}
\title{Advantage Actor-Critic with Reasoner: Explaining the Agent's Behavior from an Exploratory Perspective}

\author {
    Muzhe Guo\textsuperscript{\rm 1},
    Feixu Yu\textsuperscript{\rm 2},
    Tian Lan\textsuperscript{\rm 2}, 
    Fang Jin\textsuperscript{\rm 1}
}
\affiliations {
    \textsuperscript{\rm 1} Department of Statistics, George Washington University, United States\\
    \textsuperscript{\rm 2} Department of Electrical and Computer Engineering, George Washington University, United States\\
    \{muzheguo, fxyu, tlan, fangjin\}@gwu.edu \\
}

\usepackage{bibentry}

\begin{document}
\maketitle
\begin{abstract}
Reinforcement learning (RL) is a powerful tool for solving complex decision-making problems, but its lack of transparency and interpretability has been a major challenge in domains where decisions have significant real-world consequences. In this paper, we propose a novel Advantage Actor-Critic with Reasoner (A2CR), which can be easily applied to Actor-Critic-based RL models and make them interpretable. A2CR consists of three interconnected networks: the Policy Network, the Value Network, and the Reasoner Network. By predefining and classifying the underlying purpose of the actor's actions, A2CR automatically generates a more comprehensive and interpretable paradigm for understanding the agent's decision-making process. It offers a range of functionalities such as purpose-based saliency, early failure detection, and model supervision, thereby promoting responsible and trustworthy RL. Evaluations conducted in action-rich Super Mario Bros environments yield intriguing findings: Reasoner-predicted label proportions decrease for ``Breakout" and increase for ``Hovering" as the exploration level of the RL algorithm intensifies. Additionally, purpose-based saliencies are more focused and comprehensible.
\end{abstract}

\section{Introduction}
Reinforcement learning (RL) has demonstrated great potential in solving complex decision-making problems, enabling intelligent agents to acquire decision-making abilities through interactions with their environment. Despite its successes in diverse applications, such as gaming~\cite{mnih2013playing, silver2016mastering}, robotics~\cite{zhao2020sim}, large-scale models training~\cite{Gpt4}, and recommendation systems~\cite{afsar2022reinforcement}, RL models face criticism for their lack of transparency and interpretability. This is a challenging problem due to the dynamics of the underlying Markov Decision Process, as well as the interaction with complex environments and agent settings that shapes an agent's behavior.
Furthermore, RL models are often trained through a trial-and-error process, leading to non-intuitive and unpredictable behavior, particularly in novel scenarios. This unpredictability presents a challenge for interpreting the model's decisions, which can be especially problematic in safety-critical domains where transparency is essential. 

The field of RL interpretability is relatively new, lacking standardized approaches. As a result, researchers must explore a variety of techniques to gain valuable insights into RL decision-making processes~\cite{dazeley2023explainable}. However, the existing explainable RL approaches frequently come up short in furnishing intuitive comprehension of agents' actions. Certain explanations revolve around methodologies like tracing nodes~\cite{krarup2019model}, policy summarization~\cite{lage2019exploring}, or goal-driven interpretation~\cite{sado2023explainable}. While these approaches allow the system to articulate why a selected action aligns with achieving its goal, the resulting explanations frequently lack context-specific meaning or depth.

In this work, we propose a novel and general approach to interpreting the agent’s actions from an exploratory perspective. Our framework, named Advantage Actor-Critic with Reasoner (A2CR), consists of three interconnected networks: Policy Network ($\pi_\theta$), Value Network ($v_\omega$), and Reasoner Network ($R_\phi$). Overall, $\pi_\theta$ learns policies for selecting actions, $v_\omega$ observes the actions and evaluates states, and $R_\phi$ provides an interpretation of the agent by predicting the purpose of the actor's action, taking into account the state differences, state values, and rewards after each action. $R_\phi$ aims to collect training data, generate training labels, and train the network automatically and simultaneously by engaging with policies and actors without interfering with them. 

We provide the theoretical feasibility of computing the state shift between two states using Phase Correlation~\cite{foroosh2002extension}, theoretically supporting the autonomous generation of training labels of $R_\phi$. Additionally, we prove the statistical convergence of the proportion of training labels within $R_\phi$'s data collection process. This convergence forms the foundation for an effective collector and smooth training of $R_\phi$. 
Through experiments with Super Mario Bros games, our main findings indicate that saliencies with specific action purposes offer heightened comprehensibility. Besides, the high instability in action purposes reveals instances where agents are headed for failure. Moreover, the proportion of predicted labels unveils the degree of inherent exploration in RL algorithms.

In summary, this paper has several contributions and innovations:
(1) Our approach offers a more transparent model for the agent's decision-making process, making it applicable to a wide range of Actor-Critic algorithms without compromising the original Actor-Critic's performance.
(2) A2CR can be further combined with other interpretation techniques. To the best of our knowledge, we are the first to provide action purpose-based saliency maps for RL agents.
(3) Without manual annotation, $R_\phi$ automates its training data collection by dynamically adapting to new scenarios, and we have demonstrated the statistical convergence of its training label proportions during training.

\section{Related work}
Despite significant advancements in the field of Explainable Artificial Intelligence (XAI), achieving interpretability in deep reinforcement learning remains a challenging task, with the absence of a unified and universally applicable approach. Various approaches have been explored in this regard.
Saliency-based methods~\cite{greydanus2018visualizing, mott2019towards} aim to visualize the decision-making process of the agent via saliency maps but may highlight unwanted areas due to the agent's unobvious intentions.
Attention-based methods~\cite{wu2021self} can extract keypoints from salient regions and are not related to particular object semantics. 
The Causal Lens method~\cite{madumal2020explainable} encodes causal relationships to learn a structural causal model, facilitating interpretability, but might not be easily understandable for non-technical users.
Decision tree-based methods~\cite{vasic2019moet, liu2021learning} identify influential features driving the agent's behavior and construct decision trees. However, this approach may result in numerous nodes tracing back to the decision rule path.
Interaction data-based methods~\cite{ehsan2018rationalization, ehsan2019automated} involve collecting explanations from human players during the decision-making process to gain insights into how the agent's behavior is perceived, but annotations are costly and specific to only a single scenario.
In this work, our method serves as a plug-in module for existing Actor-Critic algorithms, enhancing their generalizability and comprehensibility, with specific experiments conducted using the Advantage Actor-Critic (A2C) approach~\cite{mnih2016asynchronous}.

\section{Methods}
Basic RL is commonly modeled as a Markov decision process (MDP), which is defined as a four-element tuple $(\mathcal{S},\mathcal{A},\mathbb{P},r)$, representing the state set, action set, state transition function, and reward function, respectively. In such an environment setting, we begin by providing an overview of both A2C and our A2CR frameworks, then focus on the A2CR, covering its interpretation goal, structure, training process, training label generation, and label properties.

\subsection{Advantage Actor-Critic Review}
The Actor-Critic approach consists of two networks: the Policy Network $\pi_\theta(a|s)$, parameterized by $\theta$, and the Value Network $v_\omega(s)$, parameterized by $\omega$. The Actor selects actions $a \in \mathcal{A}$ based on the policy, while the Critic evaluates actions using TD error to determine their effectiveness. TD error, denoted as $\delta_t$, is calculated as the difference between the TD Target $r_{t}+\gamma v_\omega(s_{t+1})$ and the current state's value $v_\omega(s_t)$, where $r_{t}$ is the reward obtained from the environment and $\gamma$ is the discount factor. $\delta_t$ of A2C serves as the advantage function, influencing the probability of selecting actions in future steps based on their effectiveness. Positive advantages increase the likelihood of choosing an action, while negative advantages decrease the probability of selection. 
The network parameters are updated at each step using $\delta_t$ as follows:
$\theta \leftarrow \theta-\beta \cdot \delta_t \cdot \partial log \pi_\theta (a_t|s_t) / \partial \theta$ for Actor and
$\omega \leftarrow \omega - \alpha \cdot \delta_t \cdot \partial v_\omega (s_t) / \partial \omega$ for Critic. $\beta$ and $\alpha$ represent the learning rates for the Actor and Critic, respectively.

\subsection{Advantage Actor-Critic with Reasoner Overview}
\begin{figure*}[t]
  \centering
  \includegraphics[width=0.85\textwidth]{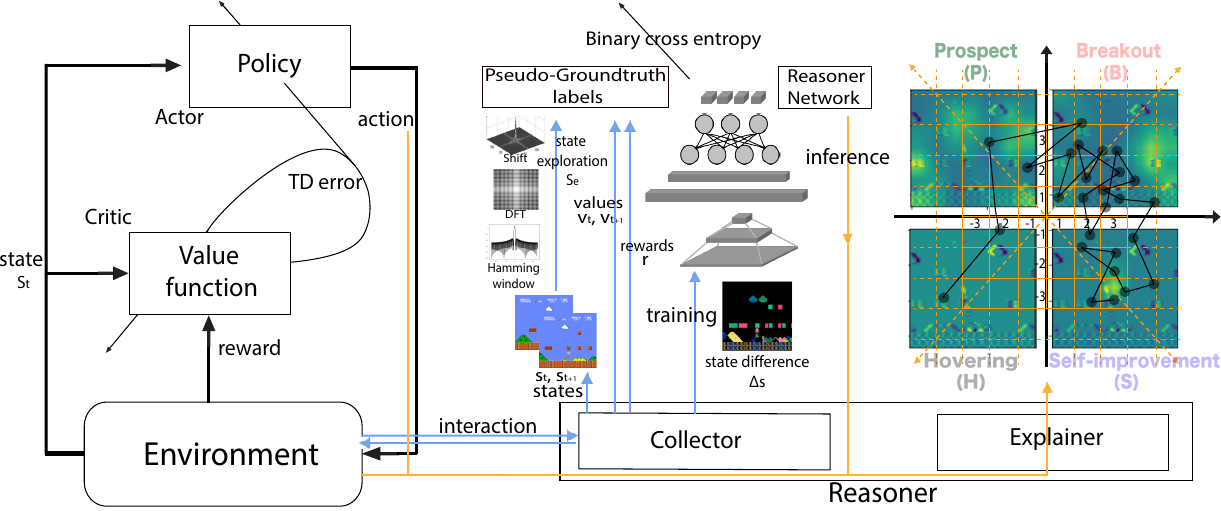}
  \caption{The proposed A2CR framework incorporates a Reasoner module to provide interpretable explanations of the agent's behavior. The blue flow lines depict the Collector's process of collecting data, generating Pseudo-Groundtruth labels, and training the Reasoner Network. The yellow flow lines represent the process and results of reasoning by the Explainer once the model is well-trained.}
  \label{fig:A2CR_Architecture}
\end{figure*}

As shown in Figure \ref{fig:A2CR_Architecture}, our proposed A2CR retains the original architecture of A2C but integrates an additional network called the Reasoner Network, parameterized by $\phi$, denoted by $R_\phi(\Delta s)$, where the input $\Delta s = s_{t+1} - s_{t}$ represents the pixel-wise difference between two adjacent state frames obtained from the A2C agent's interactions with the environment.
The Reasoner plays two roles: the Collector and the Explainer. The Collector employs separate agents, independent of the A2C agents, to collect data and train $R_\phi$ simultaneously. The Explainer utilizes the well-trained model to interpret the behavior of the A2C agents.

\subsection{Interpretation of agent actions}
To interpret an agent's actions effectively, we rely on two fundamental assumptions: (1) All actions taken by the agent are deemed rational. This assumption serves as the foundation for explaining the agent's behavior, ensuring that any derived explanation holds significance and motivation. Therefore, even when confronted with seemingly irrational behavior, we begin by assuming its underlying reasonability and subsequently explore contradictory explanations if necessary. (2) The agent exhibits a tendency toward ``laziness" or a lack of subjective motivation. This assumption aligns with the RL principles, recognizing that agents typically avoid venturing into unknown environments unless prompted by external stimuli such as rewards. It reflects the agent's inclination to remain in relatively stable environments to ensure its survival.
The two assumptions form the basis for our subsequent categorization of action purposes.

Consider an MDP denoted by $(s_t, a_t, r_t, s_{t+1})$, where at time $t$, the agent takes action $a_t$ in state $s_t$, receives a reward $r_t$, and transitions to a new state $s_{t+1}$. The difference in state values before and after the action, $v_\omega(s_{t+1}) - v_\omega(s_t)$, represents the ``indirect gain" of $a_t$, reflecting the increase in the environmental advantage enabled by the action. Additionally, the reward $r_t$ represents the ``direct gain" of $a_t$. Therefore, we define the ``total gain (G)" of $a_t$ as $w_1 \left( v_\omega(s_{t+1}) - v_\omega(s_t) \right) + (1-w_1) r_t$, where $w_1 \in [0,1]$ is a weighting parameter.
Moreover, we can measure the general change in the environment before and after an action $a_t$ through the difference between the states, termed ``state exploration ($S_e$)". In video games, $S_e$ can be obtained by comparing pixel-wise differences between two frames, encompassing areas of overlap, disappearance, and appearance. Specifically, we define $S_e$ as the sum of the Frobenius norms of the common area difference, disappeared area, and appeared area, i.e., $S_e = \lVert \text{common area difference} \rVert_{\rm F} + \lVert \text{disappeared area} \rVert_{\rm F} + \lVert \text{appeared area} \rVert_{\rm F}$.

$G$ and $S_e$ provide distinct perspectives for evaluating the impact of action $a_t$. In the context of Atari video games, we classify explanations of the agent's actions into four categories based on $G$ and $S_e$:
(1) Breakout, encoded (1,1): This category encompasses actions that result in significant state changes and substantial gains. It represents the agent making noteworthy progress in the game, such as entering a new level, defeating opponents, or achieving major milestones associated with high rewards.
(2) Self-improvement, encoded (1,0): Actions falling under this category yield significant gains with small state changes. The agent prefers such actions as they offer a relatively safe pathway to obtain rewards, such as acquiring gold, equipment, or experience in a game.
(3) Hovering (H), encoded (0,0): This category comprises actions that lead to smaller state changes and gains, often indicating mundane or conservative behavior. It includes routine actions, repetitive exploration of already known states, and stop-in-place or null actions.
(4) Prospect (P), encoded (0,1): Actions in this category result in substantial state changes but poor gains in the short term. However, they may signify the agent's exploration of the environment and preparation for potential long-term benefits. Examples include preliminary actions to gain an advantage, such as early attacks on opponents and dangerous exploration moves. Although these early actions involve risks and offer no immediate rewards, they may lead to greater benefits upon successful execution.

The Reasoner Network plays a crucial role in gaining valuable insights into the agent's decision-making process by classifying the purpose of each action of the A2C agent and conducting subsequent analysis.

\subsection{Structure of A2CR}
As mentioned earlier, A2CR consists of three interconnected networks: $\pi_\theta$, $v_\omega$, and $R_\phi$. The Actor selects actions based on learned policies, the Critic observes the actions and provides critical evaluation, and the Reasoner interprets the actions taken by the A2C agent. 
Our approach preserves the original architecture of $\pi_\theta$ and $v_\omega$ in A2C but integrates $R_\phi$ alongside them. This integration ensures that $R_\phi$ does not interfere with the other two networks' training, making it widely applicable to various Actor-Critic structures. 

$\pi_\theta(a|s)$ and $v_\omega(s)$ share the input state features obtained from three convolutional layers, each connected to a ReLU activation function. After obtaining the state features, $\pi_\theta$ utilizes two fully connected layers, followed by ReLU and softmax activation functions respectively to map the state features to a vector. This vector has the same length as the size of the action set $\mathcal{A}$ and is known as policy, which is a probability distribution. $v_\omega$ also has two fully connected layers, but omits the final softmax activation function, outputting a scale known as the state value. 
$R_\phi(\Delta s)$ has the same convolutional structure as A2C but is specifically designed to capture the features of stage differences $\Delta s$. It has an additional fully connected layer compared to $\pi_\theta$. The output of $R_\phi$ is a vector of length four, representing the probability of predictions for the four main interpretation categories. 
The Reasoner serves two roles: (1) Collector, responsible for $R_\phi$ training data collection, and (2) Explainer, utilizing the well-trained $R_\phi$ to classify the action purposes of A2C agents.

\subsubsection{Collector: data collection}
The role of the Collector is to obtain training data for $R_\phi$, while ensuring that the data collection process is synchronized with the A2C training process. 
This synchronization guarantees that $R_\phi$ accurately reflects the level of training achieved by A2C agents.
Importantly, the Collector leverages A2C networks but does not disrupt A2C training. This is achieved by employing independent agents within $R_\phi$ to interact with the environment, with their behavior controlled by the parameters of the A2C. This design ensures that the Collector can gather highly relevant data for the A2C agent while not interfering with the A2C training process. 
To ensure a diverse range of actions, the Collector employs multiple agents to collect data, which is subsequently stored in a container called Exploring Pool ($\mathcal{E}$).
Whenever new data is added to the $\mathcal{E}$, a classification label of interpretation is automatically calculated based on the existing data within the pool. This label, known as the ``Pseudo-Groundtruth" label, serves as the supervised label for $R_\phi$ training.  
$\mathcal{E}$ has a fixed capacity $N$ and is constantly updated during A2CR training, with old data being continuously dropped to accommodate new data.

\subsubsection{Explainer: action explanation}
The Explainer component of Reasoner aims to classify the purpose and provide interpretations of each action undertaken by the A2C agent. This is achieved by training $R_\phi(\Delta s)$ to perform a classification task using the state difference $\Delta s = s_{t+1} - s_t$ as inputs. The resulting classification enables various interpretation and inference tasks, including (1) saliency maps based on the agent's action purposes, (2) early detection of agent failures, (3) identification of the agent's level of exploration, and (4) metrics for RL training completion. By leveraging the classification results of the well-trained $R_\phi$, we gain valuable insights into the agent's behavior and decision-making process, thereby enhancing the understanding of Actor-Critic RL models.

\subsection{A2CR training} 
\label{sec:A2CR-training}

As the training of $R_\phi$ relies on the data collected by the Collector, which may initially be of lower quality due to the underfitting in the Policy and Value networks, it would be computationally inefficient to train $R_\phi$ at this early stage. To address this, we propose an optional strategy of setting a predetermined critical number of frames, such as 4/5 of the total number, before initiating the training of $R_\phi$. This ensures that $R_\phi$ training begins only when the A2C networks have achieved a suitable level of proficiency. The A2CR training process is depicted in Algorithm \ref{alg:trainingA2CR}.

\subsubsection{Training of Actor and Critic}
The training process for the Actor and Critic in A2CR follows the methodology employed in the original A2C approach. The TD error and policy gradient serve as the foundation for updating $v_{\omega}$ and $\pi_{\theta}$, respectively. To strike a balance between the critic loss ($L_c$) and actor loss ($L_a$), we introduce a weighting parameter $\rho_1$. Additionally, we incorporate an entropy bonus ($H$) with a parameter $\rho_2$ into the total loss. This entropy bonus encourages the policy entropy to increase, enhancing the model's exploratory capabilities and preventing premature convergence to local optima. The total loss can be expressed as $L = L_a + \rho_1 L_c - \rho_2 H$.

\subsubsection{Training of Reasoner}
The parameters of $R_{\phi}$ are updated regularly based on the data available in the Exploring Pool, which is continuously refreshed during the training process. This constant updating of the pool, along with its fixed capacity, ensures a smooth training process. For the loss function, we employ Binary Cross Entropy (with logits) for Multi-Class classification. Specifically, assuming that the Pseudo-Groundtruth label has $M$ classes, $R_{\phi}$ predicts the probability of a given data sample belonging to the $j$-th class as $p_j$. Using one-hot encoding, we represent the $j$-th element of the Pseudo-Groundtruth label as $y_j$. The Reasoner loss can be expressed as: $L_r = -\sum_{j=1}^M y_j \log(s(p_j))$, where $s(x) = 1/(1+e^{-x})$ denotes the sigmoid function.

\begin{algorithm}[H]
\caption{Training of Advantage Actor-Critic with Reasoner}
\label{alg:trainingA2CR}
\textbf{Initialize} Policy Network $\pi_{\theta} (a|s)$, Value Network $v_{\omega} (s)$, and Reasoner Network $R_{\phi} (\Delta s)$\\
\textbf{Initialize} the Exploring Pool $\mathcal{E}$ with a capacity of $N$\\
\textbf{Parameter}: weights $\rho_1$, $\rho_2$, $w_1$, and discount factor $\gamma$ \\
\textbf{Output}: Well-trained A2CR 
    \begin{algorithmic}[1] 
    \FOR{\texttt{episode = 1, ..., M}}
        \STATE For A2C agent: observe initial state $s_0$
        \FOR{\texttt{t = 1,..., T}}
            \STATE Sample action $a_t \sim \pi_{\theta} (a|s_t)$ 
            \STATE Calculate the entropy of distribution determined by $\pi_{\theta}$:  $H = -\sum_{x \in \pi} p(x) log p(x)$ 
            \STATE Obtain reward $r_t$ and new state $s_{t+1}$
            \STATE Calculate TD target: $y_t = r_t + \gamma  v_{\omega} (s_{t+1})$ 
            \STATE Get Advantage (TD error): $\delta_t = y_t - v_{\omega} (s_{t})$ 
            \STATE Critic loss: $L_c = \delta_t^2$ 
            \STATE Actor loss: $L_a = -log \pi_\theta (a_t|s_t) \delta_t$ 
            \STATE Update $\omega$ and $\theta$ by minimizing total loss $L= L_a + \rho_1 L_c  - \rho_2 H$
        \ENDFOR
        \STATE  For Reasoner agent: observe initial state $s_0$ \\
        (optional: start only if \texttt{episode} $>$ \texttt{4M/5})
        \FOR{\texttt{t = 1,..., T}}
            \STATE Sample action $a_t \sim \pi_{\theta} (a|s_t)$ 
            \STATE Obtain reward $r_t$ and new state $s_{t+1}$ 
            \STATE  State difference: $\Delta s_t = s_{t+1} - s_t$ 
            \STATE  State value difference: $\Delta v_t = v_{\omega} (s_{t+1}) - v_{\omega} (s_{t})$ 
            \STATE Calculate the total gain value:  \\
            $G_t = w_1(v_{\omega} (s_{t+1})-v_{\omega}(s_{t})) + (1-w_1) r_t$
            \STATE  Calculate the state exploration value $S_{e,t}$ using phase correlation 
            \STATE  Sample half of the data randomly from the exploring pool: $\mathcal{E}_0 \sim \mathcal{E}$, $\mathcal{E}_0 \subsetneq \mathcal{E}$  
            \STATE  Get Pseudo-Groundtruth label: \\ $l_{t} = (\mathbb{1}_{G_t \geq  2\sum_{i\in \mathcal{E}_0} G_{i}/N}, \mathbb{1}_{S_{e,t} \geq  2\sum_{i\in \mathcal{E}_0} S_{e,i}/N})$
            \STATE  Add $(G_t, S_{e,t}, l_t)$ to $\mathcal{E}$ and update $\mathcal{E}$ at regular intervals
            \STATE  Reasoner loss: $L_r=$ cross entropy with logits using the data ($\Delta s_t$, $l_t$)
            \STATE  Update $\phi$ by minimizing $L_r$
        \ENDFOR
    \ENDFOR
    \end{algorithmic}
\end{algorithm}

\subsection{Pseudo-Groundtruth labels} 
We utilize the data collected by the Collector to provide supervised training for $R_\phi$, with the labels of the collected data dependent on factors $G$ and $S_e$, both of which require calculation. The calculation of $G$ is straightforward: 
$$G_t = w_1(v_{\omega} (s_{t+1})-v_{\omega}(s_{t})) + (1-w_1) r_t$$
On the other hand, acquiring $S_e$ necessitates measuring the shift between adjacent states. In determining the amount of shift, we employ the principle of Phase Correlation. Given two states $s_{t+1}$ and $s_t$ with dimensions $W\times H$ pixels, we assume they undergo a relative shift denoted by $\Delta x$ along the x-axis and $\Delta y$ along the y-axis. To reduce boundary effects, we first apply the Hamming window to the states. This allows us to approximate the circular shift of the two states:
$$s_{t+1}(x,y) \approx  s_{t}((x-\Delta x) \ \text{mod}\ W, (y-\Delta y) \ \text{mod}\ H )$$
Using the shift theorem of 2D Discrete Fourier Transform (DFT), we derive the following equation:
$$\mathfrak{F}_{s_{t+1}}(x,y) = \mathfrak{F}_{s_{t}}(x,y) e^{-2\pi i (\frac{x\Delta x}{W} + \frac{y \Delta y}{H})}$$
From this, the cross-power spectrum can be expressed as:
\begin{align*} 
R(x,y) 
& = 
\frac{\mathfrak{F}_{s_{t}}(x,y) \circ \overline{\mathfrak{F}_{s_{t+1}}}(x,y)}{|\mathfrak{F}_{s_{t}}(x,y) \circ \overline{\mathfrak{F}_{s_{t+1}}}(x,y)|} \\
& = 
\frac{\mathfrak{F}_{s_{t}}(x,y) \circ \overline{\mathfrak{F}_{s_{t}}}(x,y)e^{2\pi i (\frac{x\Delta x}{W} + \frac{y \Delta y}{H})} }{|\mathfrak{F}_{s_{t}}(x,y) \circ \overline{\mathfrak{F}_{s_{t}}}(x,y)e^{2\pi i (\frac{x\Delta x}{W} + \frac{y \Delta y}{H})}|} \\
& = 
e^{2\pi i (\frac{x\Delta x}{W} + \frac{y \Delta y}{H})}.
\end{align*} 
Next, we introduce the 2D shifted impulse function defined as $\delta(x-a,y-b) \stackrel{\text{def}}{=} \delta(x-a)\delta(y-b)$ where $a$ and $b$ are arbitrary constants. This function can be obtained as the product of two Dirac-Delta functions $\delta(x-a) = \lim_{n\to \infty} n\mathbb{1}_{[-\frac{1}{2n} \leq x-a \leq \frac{1}{2n}]}$ and $\delta(y-b) = \lim_{n\to \infty} n\mathbb{1}_{[-\frac{1}{2n} \leq y-b \leq \frac{1}{2n}]}$. Applying the Fourier transform to $\delta(x-a,y-b)$ yields:
\begin{align*} 
\mathfrak{F}\left(\delta(x-a,y-b)\right)
= &
\underset{R^2}{\iint} \delta(x-a,y-b) e^{-i2\pi(x \xi_x + y\xi_y)} dxdy \\
= & 
e^{-i2\pi(a \xi_x + b\xi_y)},
\end{align*} 
where $\xi_x$ and $\xi_y$ are oscillation frequencies along x and y axes. This leads to 
$\mathfrak{F}^{-1}\left(e^{-i2\pi(a \xi_x + b\xi_y)} \right) = \delta(x-a,y-b)$.
Thus, the inverse Fourier transform of the cross-power spectrum $R(x,y)$ is: 
$$
\mathfrak{F}^{-1}\left(R(x,y) \right) =
\mathfrak{F}^{-1}\left( e^{i2\pi (\frac{x\Delta x}{W} + \frac{y \Delta y}{H})} \right) = \delta(x+\Delta x , y+\Delta y).
$$
This indicates that we can obtain the shifts $\Delta x$ and $\Delta y$ by finding the location of the peak in the above results: 
$(\Delta x, \Delta y) = \mathrm{argmax}_{x,y} \left( \mathfrak{F}^{-1}\left(R(x,y) \right) \right)$. In practice, the fast Fourier transform (FFT) can be used to efficiently compute the discrete Fourier transform and its inverse. Once $\Delta x$ and $\Delta y$ are obtained, we can get the corresponding value of factor $S_e$ using the following relationship:
\begin{align*} 
S_{e,t} =  &  \underbrace{ \| s_{t+1} {\scriptstyle [0:W-\Delta x,0:H-\Delta y]} - s_{t} {\scriptstyle [\Delta x:W, \Delta y:H]} \|_F }_{\text{common area difference}}   \\
& + 
\underbrace{ \| s_{t}{\scriptstyle [0:W, 0:\Delta y]} + s_{t}{\scriptstyle [0:\Delta x, 0:H]} -  s_{t}{\scriptstyle [0:\Delta x, 0:\Delta y]}  \|_F }_{\text{disappeared area}}\\
& + 
\|  s_{t+1}{\scriptstyle [W-\Delta x:W, 0:H]} + s_{t+1}{\scriptstyle [0: W, H-\Delta y:H]} \\
& \quad \underbrace{ \ - 
s_{t+1}{\scriptstyle [W-\Delta x:W, H-\Delta y:H]} \|_F \quad \quad \quad \quad \quad}_{\text{appeared area}}
\end{align*} 

\begin{figure}[t]
  \centering
  \includegraphics[width=0.5\textwidth]{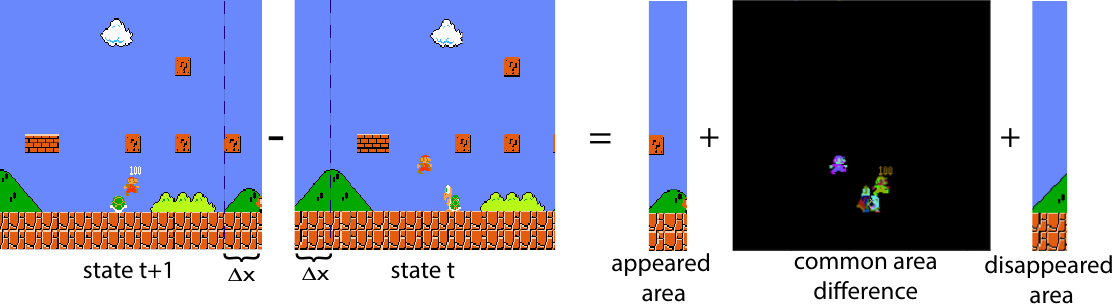}
  \caption{The state exploration $S_e$ is computed as the sum of the Frobenius norms of three areas: the area where the new state appears, the area where the old state disappears, and the difference between the common areas of both states. It's worth noting that in Super Mario Bros, $\Delta y$ is always zero.}  
  \label{fig:stage_exploration}
\end{figure}
In the above relation, $\|.\|_F$ is the Frobenius Norm and we assume that $s_{t+1}$ is shifted to the upper right relative to $s_{t}$ with $\Delta x\geq0$ and $\Delta y\geq0$. 
For the Open AI Super Mario Bros, we observed that $\Delta y$ is always zero. Figure \ref{fig:stage_exploration} illustrates the procedure for obtaining the value of $S_e$ in this case.

After computing $G_{t}$ and $S_{e,t}$, we need to assign a Pseudo-Groundtruth label to them. 
Suppose the current Exploring Pool $\mathcal{E}$ with a capacity of $N$ contains the data tuples $\{(G_{i}, S_{e,i}, l_{i})\}_{i=t-N}^{t-1}$. 
We start by randomly sampling half of the data to shuffle their order, which yields a subset $\mathcal{E}_0$, then the 2D 0-1 encoded label for $(G_{t},S_{e,t})$ is given by:
$$
l_{t} = \left(\mathbb{1}_{G_t \geq  2\sum_{i\in \mathcal{E}_0} G_{i}/N},  \ 
\mathbb{1}_{S_{e,t} \geq  2\sum_{i\in \mathcal{E}_0} S_{e,i}/N} \right),
$$
where $\mathbb{1}$ is the indicator function. Once $l_t$ is obtained, the tuple $(G_{t},S_{e,t}, l_t)$ is stored in the $\mathcal{E}$, and the oldest tuple with an index of $i=t-N$ is removed to maintain the pool's up-to-date status.

\subsection{Convergence of Pseudo-Groundtruth labels}
An essential attribute of a quality Collector is its ability for collecting labels that accurately represent the environment in which it operates. In an MDP, where the future evolution of Markov processes is independent of the past, it is possible to consider two random variables representing the same feature collected by the agents before and after a sufficiently long interval as approximately independent. 
For example, if a feature $X_0$ is generated by agents and follows the distribution $F_0(x)$ in the global environment, throughout the long training process, this feature can be modeled as a sequence of random variables $X_1, ..., X_n$ at different update moments of $\mathcal{E}$. The variables $X_1, ..., X_n$ are considered to be nearly independent and exhibit approximate convergence towards $X_0$. The data in $\mathcal{E}$ at each update time can be treated as a random sample from the corresponding $X_i$.
In Theorem \ref{theorem_prop} and its corollary, we introduce the convergence of the proportion of labels in $\mathcal{E}$ in the context of Reasoner training.

\begin{theorem}
\label{theorem_prop}
Assume $X_1, ..., X_n :\Omega \rightarrow \mathcal{R}$ are independent random variables defined on a common probability space $(\Omega, \mathcal{F},\mathcal{P})$. They have continuous cumulative distribution function $F_i(x)$ and expectations $E(X_i) = \mu_i<\infty$. Let $\{(x_i^{(1)}, x_i^{(2)}, ..., x_i^{(N)})\}_{i=1}^n$ be a sequence of samples of $\{X_i\}_{i=1}^n$ with a sequence of mean values$\{\bar{x}_i\}_{i=1}^n$. The corresponding label sequence is defined as $\{L_i | L_i = \mathbb{1}_{X_0 \geq \bar{x}_i}\}_{i=1}^n$, where $X_0 \sim F_0(x)$ with expectation $\mu_0$. If $\{F_i(x)\}_{i=1}^n$ converge pointwise to $F_0$ and $\{\mu_i(x)\}_{i=1}^n$ converge to $\mu_0$ in probability, then $\sum_{i=1}^n L_i/n$ converges to $F_0(\mu_0)$ in probability.
\end{theorem}

\begin{proof}
Define the empirical distribution function $\hat{F}_i(t)=\frac{1}{n} \sum_{j=1}^N \mathbb{1}_{x_i^{(j)} \leq t}$, $i=1,2,...,n$, then $L_i \sim$ Bernoulli$(\hat{F}_i(\bar{x}_i))$ with mean $E(L_i) = \hat{F}_i(\bar{x}_i)$ and variance $var(L_i) = \hat{F}_i(\bar{x}_i) (1 - \hat{F}_i(\bar{x}_i))$. 
Let $S_n = \sum_{i=1}^n L_i$, According to Kolmogorov's inequality, we obtain:
$$
P\left(\max_{1 \leq n \leq K} \Big| S_n - \sum_{i=1}^n E(L_i) \Big| \geq \lambda \right) \leq \frac{1}{\lambda^2} \sum_{i=1}^K var(L_i), \quad 
$$
$$
\forall \lambda > 0, K\geq n
$$
Setting $K=n$ and using the constraints $0 \leq \hat{F}_i(t) \leq 1$,
\begin{align*} 
P\left( \Big| \frac{S_n}{n} - \frac{1}{n}\sum_{i=1}^n \hat{F}_i(\bar{x}_i) \Big| \geq \lambda \right) 
& \leq  
\frac{1}{\lambda^2 n^2} \sum_{i=1}^n \hat{F}_i(\bar{x}_i) (1 - \hat{F}_i(\bar{x}_i)) \\
& \leq  \frac{1}{4n\lambda^2},
\end{align*}
which implies that: $\frac{S_n}{n} - \frac{1}{n}\sum_{i=1}^n \hat{F}_i(\bar{x}_i) \overset{p}{\to} 0$. By the strong law of large numbers,  $\hat{F}_i(t)$ converges to $F_i(t)$ as $N\rightarrow \infty$ almost surely for every $t$. In addition, simple algebra by Taylor's theorem yields:
$$\frac{1}{n}\sum_{i=1}^n F_i(\bar{x}_i) = \frac{1}{n}\sum_{i=1}^n \Big( F_i(\mu_i) + F_i'(\mu_i)(\bar{x}_i - \mu_i) + o(\bar{x}_i - \mu_i) \Big) $$
Combining the above results, taking into account the property $\bar{x}_i - \mu_i = O(1/\sqrt{N})$ = $o(1) \overset{p}{\to} 0$, we can arrive at:
$$\frac{S_n}{n} \overset{p}{\to} \frac{1}{n}\sum_{i=1}^n \hat{F}_i(\bar{x}_i) \overset{a.s.}{\to} \frac{1}{n}\sum_{i=1}^n F_i(\bar{x}_i) \overset{p}{\to} \frac{1}{n}\sum_{i=1}^n F_i(\mu_i)$$
Finally, we know that for any $\epsilon>0$,
\begin{align*}
& P(\Big|\frac{1}{n}\sum_{i=1}^n F_i(\mu_i)-F_0(\mu_0)\Big|>\epsilon) \\
\leq & 
P(\frac{1}{n}\sum_{i=1}^n \Big|F_i(\mu_i)-F_0(\mu_i)\Big|>\frac{\epsilon}{2}) \\
& + 
P(\frac{1}{n}\sum_{i=1}^n \Big|F_0(\mu_i)-F_0(\mu_0)\Big|>\frac{\epsilon}{2})
\end{align*}
where $P(\frac{1}{n}\sum_{i=1}^n \Big|F_i(\mu_i)-F_0(\mu_i)\Big|>\frac{\epsilon}{2}) {\to} 0$ can be derived from pointwise convergence of $F_i$, $P(\frac{1}{n}\sum_{i=1}^n \Big|F_0(\mu_i)-F_0(\mu_0)\Big|>\frac{\epsilon}{2}) {\to} 0$ is the result of convergence in probability under the continuous mapping theorem. Thus we proved $\frac{1}{n}\sum_{i=1}^n F_i(\mu_i) \overset{p}{\to} F_0(\mu_0)$, leading to $S_n/n \overset{p}{\to} F_0(\mu_0)$.
\end{proof}

As the training progresses, the feature values collected from both $G$ and $S_e$ progressively align with the global settings and feature distributions become more similar to that of the global environment. This is attributed to the agents' improved performance and expanded exploration of maps, providing empirical support for the assumptions in Theorem \ref{theorem_prop}.
For a more general case, assuming that we partition the interpretability of the agent's behavior using $d$ independent variables, the label dimension would be $d$. This leads us to derive the following Corollary \ref{corollary_prop}.

\begin{corollary}
\label{corollary_prop}
Assume $\{X_{ij}\}^{\scaleto{1 \leq i \leq n}{5pt}}_{\scaleto{1 \leq j \leq d}{5pt}}
:\Omega \rightarrow \mathcal{R}$ are n$\times$d independent random variables defined on a common probability space $(\Omega, \mathcal{F},\mathcal{P})$. They have continuous cumulative distribution function $F_{ij}(x)$ and expectations $E(X_{ij}) = \mu_{ij}<\infty$. For any $1 \leq j \leq d$, Let $\{(x_{ij}^{(1)}, x_{ij}^{(2)}, ..., x_{ij}^{(N)})\}_{i=1}^n$ be a sequence of samples of $\{X_{ij}\}_{i=1}^n$ with a sequence of mean values$\{\bar{x}_{ij}\}_{i=1}^n$. 
Define the $d$-dimension label sequence as 
$\{\vec{L_{i}}=(L_{i1},  L_{i2}, ... , L_{id}) | L_{ij} =  \mathbb{1}_{X_{0j} \geq \bar{x}_{ij}}\}_{i=1}^n$, where $X_{0j} \sim F_{0j}(x)$ with expectations $\mu_{0j}$.
If $\{F_{ij}(x)\}_{i=1}^n$ converge pointwise to $F_{0j}$ and $\{\mu_{ij}(x)\}_{i=1}^n$ converge to $\mu_{0j}$ in probability, then 
$\#(l_1, l_2, ... , l_d)/n$ converges to $\Pi_{j=1}^d (F_{0j}(\mu_{0j}))^{l_j} (1-F_{0j}(\mu_{0j}))^{1-l_j}$ in probability, where $\#$ is symbol for counting and $l_j \in \{0,1\}$.
\end{corollary}

\begin{proof}
Consider the function T(k) given by:
$T(k) = \#(\underbrace{ 0\lor 1, \dots , 0\lor 1}_{k-1 \ \text{times}},l_k,l_{k+1}, \dots ,l_d)/n$. 
Using Theorem \ref{theorem_prop} and the condition that the $d$ elements in each label $\vec{L_{i}}$ are independent, we can get:
\begin{align*} 
T(1) = &\frac{\#(l_1, l_2, \dots , l_d)}{n} \\
= &  
\left( \frac{\#(1, l_2, \dots , l_d)}{n} \right)^{l_1}  \left( \frac{\#(0, l_2, \dots , l_d)}{n} \right)^{1-l_1} \\ 
= &  
\left( \frac{\#(1, l_2, \dots , l_d)}{\#(0\lor1, l_2, \dots , l_d)} 
\frac{\#(0\lor1, l_2, \dots , l_d)}{n}  \right)^{l_1} \\
&
\times \left( \frac{\#(0, l_2, \dots , l_d)}{\#(0\lor1, l_2, \dots , l_d)}
\frac{\#(0\lor1, l_2, \dots , l_d)}{n} \right)^{1-l_1} \\
\overset{p}{\to} & \big(F_{01}(\mu_{01})  T(2) \big)^{l_1} 
\big(1-F_{01}(\mu_{01})  T(2) \big)^{1-l_1} \\
= &  \big(F_{01}(\mu_{01})\big)^{l_1} \big(1-F_{01}(\mu_{01})\big)^{1-l_1} T(2) \\
\overset{p}{\to} & \dots    
\overset{p}{\to}  \Pi_{j=1}^{d-1} (F_{0j}(\mu_{0j}))^{l_j} (1-F_{0j}(\mu_{0j}))^{1-l_j} T(d) \\
\overset{p}{\to} &  
\Pi_{j=1}^{d} (F_{0j}(\mu_{0j}))^{l_j} (1-F_{0j}(\mu_{0j}))^{1-l_j}\qedhere
\end{align*} 
\end{proof}

Since our experiments involve two features $G$ and $S_e$, this corresponds to a specific case with $d=2$.

\section{Evaluation}
To evaluate our model's performance in an action-rich environment and also ensure readers' easy comprehension, we use eight distinct maps from the Super Mario Bros games~\cite{gym-super-mario-bros} as examples. All games use ``complex movement" action settings and are with an OpenAI Gym interface. Details of model configurations can be found in the Appendix.

\subsection{State saliency maps on agent intents}
\begin{figure*}[h]
  \centering
  \includegraphics[width=1.01\textwidth]{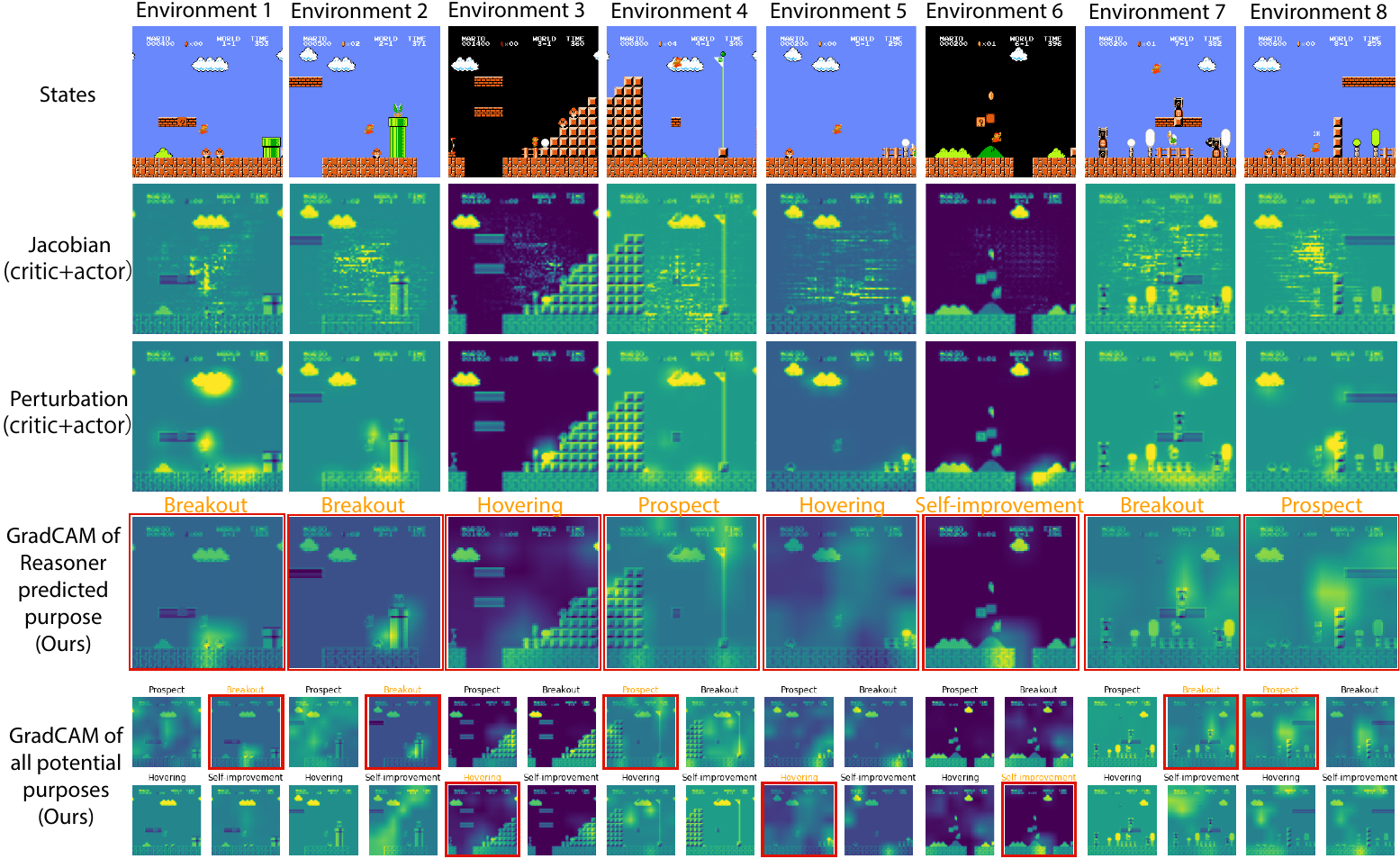}
  \caption{Comparison of Jacobian saliency, perturbation-based saliency to our purpose-driven saliency. One of the greatest strengths of our approach is that the saliency maps can be generated according to the underlying purpose of the agent's action.} 
  \label{fig:baselines_vs_ours}
\end{figure*}
The Reasoner Network can be effectively combined with visual explanation techniques to provide state saliency maps that reveal the key features and regions influencing the agent's intent during the decision-making process. By employing GradCam~\cite{selvaraju2017grad} on the Reasoner Network, the last two rows of Figure~\ref{fig:baselines_vs_ours} showcase the saliency maps that highlight important areas corresponding to the agent's intended purposes. These purposes include the one identified by the Reasoner's classification (borders outlined in red) as well as the other three potential purposes. Compared to the traditional methods of understanding agents such as Jacobian saliency ~\cite{wang2016dueling, zahavy2016graying} and Perturbation-based approach ~\cite{greydanus2018visualizing}, one of the greatest strengths of our approach is that the saliency maps can be generated according to the purpose of the agent’s actions.

Our saliency information not only informs an RL agent's decisions about ``important regions" but also offers more profound insights into how these ``important regions" serve specific and distinct purposes in the agent's actions. A clear example can be observed in the first column, where the regions highlighted by gradient-based Jacobian saliency lack human interpretability. Although Perturbation saliency highlights ``important" regions more effectively, humans still struggle to comprehend why certain highlighted regions, such as the cloud at the top of the frame, are considered important, which is counterintuitive. In contrast, our approach initially predicts that the agent's current action is ``Breakout" and then highlights only the area with the mushroom as crucial for this specific purpose. Another intriguing observation can be made in the examples in the third and fifth columns, where the Saliency maps obtained by our method for the ``Hovering" purpose are notably broader, aligning well with the expected nature of the agent's action for this purpose.

The saliency maps representing alternative potential action purposes in the last row are intended as auxiliary information rather than the primary focus. They allow us to explore how saliency changes under different action purposes. It's important to note that the absence of underlying features for these potential purposes may occasionally result in less clear saliency patterns, but in the majority of cases, these maps offer meaningful insights for human interpretation. These saliency maps not only provide user-friendly insights into the decision-making process and interactions of RL agents with their environment but also aid domain experts in comprehending the reasoning behind the Reasoner Network's classifications. 

\subsection{AI safety: before the failure}
\begin{figure}[h]
\centering
\includegraphics[width=0.3\textwidth]{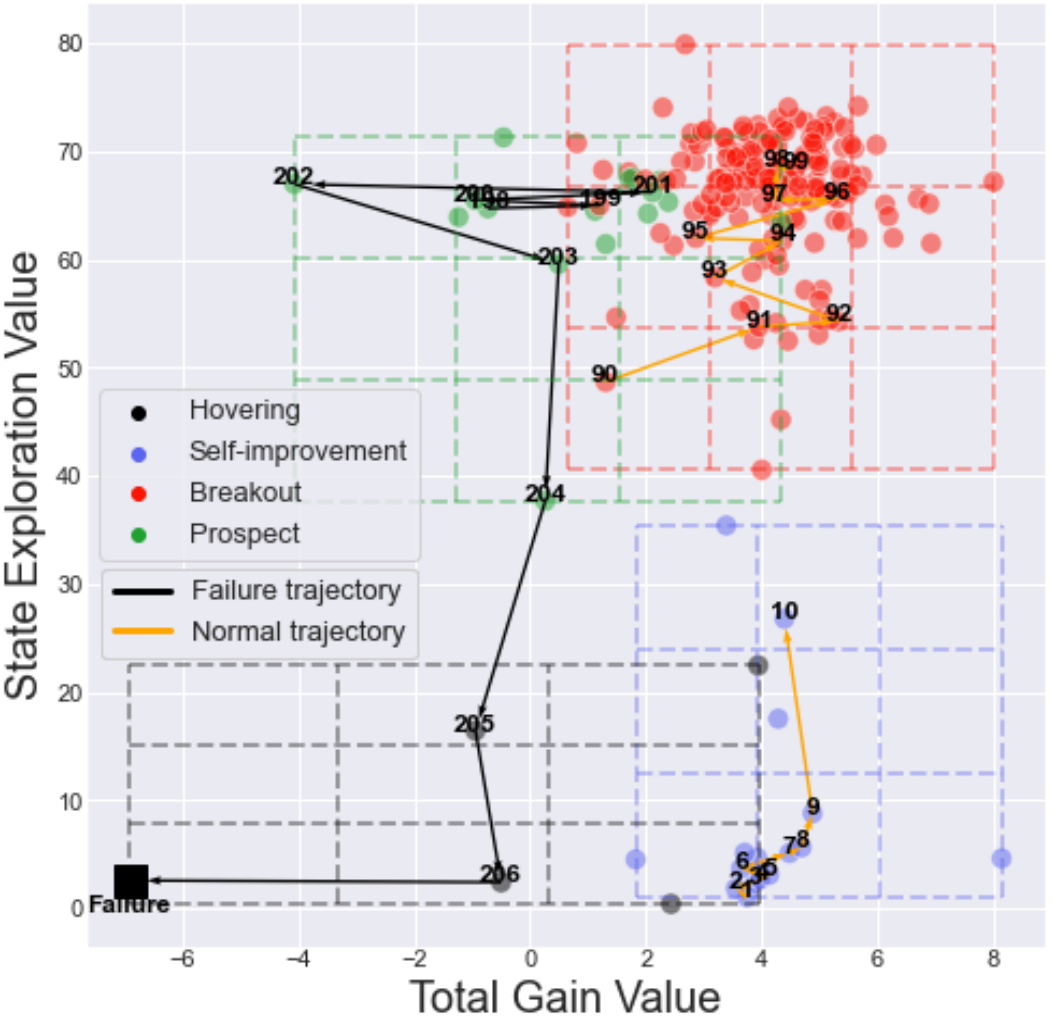}
  \caption{Action classification and early failure detection}
 \label{fig:AI_safety}
\end{figure}
The Reasoner Network plays a crucial role in classifying the actions taken by the agent. Figure~\ref{fig:AI_safety} illustrates the classification results of the agent's actions in a final failed episode, providing valuable insights into the agent's behavior throughout the episode. In RL, agent failure is a common occurrence, especially when the model is not well-trained. Our approach presents a novel method for detecting potentially dangerous behavior in the agent, making contributions to the field of AI safety. Notably, we have discovered that substantial fluctuations in the continuous action classification in the short term indicate the agent's action purpose or behavior is becoming unstable, which may ultimately lead to failure. This finding offers new insight into ensuring AI safety, as it allows for early detection and prevention of dangerous behaviors in RL agents.

\subsection{Identifying the exploration encouragement}
\begin{figure}[h]
\centering
    \includegraphics[width=0.4\textwidth]{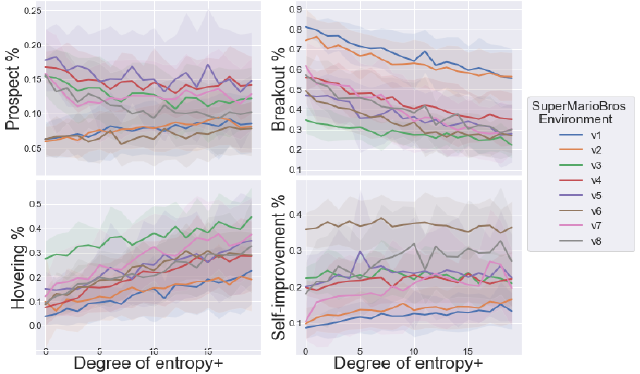}
  \caption{The relationship between policy entropy increment and the ratio of Reasoner-provided labels}
 \label{fig:Exploration_identification}
\end{figure}
Our proposed framework enables the identification of the level of exploration encouragement within RL models. We achieve this by introducing an entropy increment on the well-trained A2C policy and examining its relationship with the ratio of the four labels provided by $R_{\phi}$. To increment the entropy of the policy distribution, we employ a specific method. At each time step, we uniformly add a value of 0.001 to the policy distribution, which consists of 12 discrete values. Subsequently, we normalize the policy distribution to ensure it remains a valid probability distribution. The degree of entropy increment, denoted as entropy+, is determined by the number of such additions. We conducted 100 episodes for each degree of entropy+ per environment. As depicted in Figure \ref{fig:Exploration_identification}, we observe that as the entropy of the policy increases, there is a gradual decrease in the proportion of ``Breakout" labels, accompanied by an increase in the proportion of ``Hovering" labels. On the other hand, ``Self-improvement" and ``Prospect" labels show no clear trend. The observed changes in the proportions of ``Breakout" and ``Hovering" labels align with our expectations, as providing more exploration encouragement to the agent allows it to explore the environment rather than directly selecting an optimal strategy. In practice, by utilizing the ratio of labels provided by $R_{\phi}$, we are able to estimate the level of exploration encouragement in RL algorithms. This identification of the exploration hierarchy empowers us to proactively select the most suitable exploration strategy algorithm for agents across different environments.

\subsection{Convergence of Pseudo-Groundtruth labels}
\begin{figure}[h]
\centering
    \includegraphics[width=0.3\textwidth]{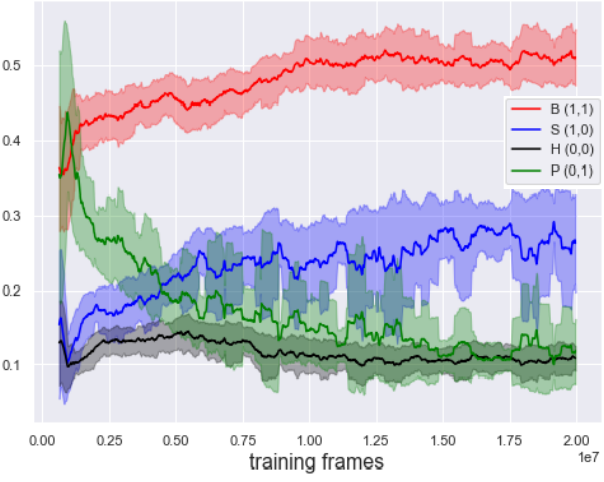}
  \caption{Convergence of four Pseudo-Groundtruth labels}
 \label{fig:convergence}
\end{figure}
We record the proportions of Pseudo-Groundtruth labels throughout $R_{\phi}$ training process to support the conclusions of Corollary~\ref{corollary_prop} for the $d=2$ case with the assumption $G$ and $S_e$ are two independent features. The trends are shown in Figure~\ref{fig:convergence}, where it can be observed that all four label proportions eventually converge to relatively stable values. Specifically, we observe that the proportions of ``Breakout" and ``Self-improvement" rise initially before leveling off, while the proportion of ``Prospect" experiences a decline to reach a stable state. Additionally, the proportion of ``Hovering" exhibits minor fluctuations before ultimately settling into a consistent pattern. The convergence of training label proportions not only highlights the evolving trend of the agent's actions during the training process but also serves as an indicator of the completion of A2C model training.

\section{Conclusion}
Incorporating the Reasoner module into the Actor-Critic framework significantly enhances its interpretability, during both the training process and the inference process. Through theoretical proof and experimental validation, we have demonstrated the effectiveness of our proposed A2CR framework, which offers valuable contributions in the areas of visual interpretation, failure early detection, RL algorithm exploration identification, and RL supervision. By providing a more interpretable RL framework, our goal is to  improve understanding, trust, and practical deployment of RL models in a wide range of scenarios.

\bibliography{arxiv}

\onecolumn
\section{Appendix}
\label{Appendix}
\subsection{Setting and Scores in Training}
Our objective is to offer an interpretation of the agent's behavior in reinforcement learning, allowing and acknowledging situations where the agent may not successfully complete the game. Based on our observations, setting the value of $\rho_2$ to be between 0.1 and 0.3 has been found to facilitate the agent's learning and successful completion. However, we intentionally set $\rho_2$ to 0.5 to promote behavioral diversity and increase the likelihood of the agent failing the game.
The Frame stack is configured with a value of 3 to facilitate its integration with the Grad-CAM technique. It is recommended to set the discount factor $\gamma$ between 0.9 and 0.99 and set the parameter $\rho_1$ between 0.3 and 1. Additionally, $w_1$ represents the relative importance of the state value difference ($\Delta v$) and the reward ($r$) in the feature “total gain” value ($G$). Considering both factors to be equally significant, we set $w_1$ to 0.5.

\begin{table}[h]
\label{parameters}
\centering
  \caption{Parameter settings for the A2CR training}
  \begin{tabular}{cc}
    \toprule
Environment  & SuperMarioBros\\
World  &  \{1, 2, 3, 4, 5, 6, 7, 8\} \\
Version  &  0 (standard ROM)\\
Movement & COMPLEX MOVEMENT \\
Warp frame size & 84 $\times$ 84 \\
Frame stack & 3 \\
$\gamma$ & 0.9 \\
$\rho_1$ & 0.5 \\
$\rho_2$ & 0.5 \\
$w_1$ & 0.5 \\
Total number of training frames for A2C  & 2e+7 \\
Total number of training frames for Reasoner & 2e+7 \\
Batch size & 16 \\
Learning rate for A2C & 2.5e-4\\
Learning rate for Reasoner & 2.5e-4 \\
Number of A2C workers & 16 \\ 
Number of Reasoner workers (collector) & 4 \\ 
Exploring Pool size (N) & 1000 \\
Number of Reasoner classification (M) & 4\\
\bottomrule
\end{tabular}   
\end{table}      

Based on the provided configurations, Table \ref{score_and_status} presents the average scores of the agent over the last 100 episodes during the training phase. According to the tests, the agent is able to successfully pass the SuperMarioBros World  $\{1, 3, 4, 5, 6, 7\}$ However, in the case of World $\{2\}$, it is challenging for the agent to pass it. Furthermore, in World $\{8\}$, it is observed that the agent consistently fails in the final steps leading up to completing the game.
\begin{table}[h]
\centering
  \caption{Average score for the last 100 training episodes}
  \begin{tabular}{ccc}
\toprule
Environment & score & pass/fail status \\
\midrule
SuperMarioBros-1-1-v0 & 2591.29  & pass\\
SuperMarioBros-2-1-v0 & 1355.30  & fail\\
SuperMarioBros-3-1-v0 & 2175.47  & pass\\
SuperMarioBros-4-1-v0 & 2952.64  & pass\\
SuperMarioBros-5-1-v0 & 1849.54  & pass\\
SuperMarioBros-6-1-v0 & 2632.66  & pass\\
SuperMarioBros-7-1-v0 & 1916.08  & pass\\
SuperMarioBros-8-1-v0 & 2283.60  & fail (close to pass)\\
\bottomrule
\label{score_and_status}
\end{tabular}   
\end{table}      


\newpage
\subsection{Exploration Encouragement Level and Proportions of Reasoner-Provided Labels}
\begin{table}[!htbp]
\centering
\caption{The relationship between policy entropy increment (entropy+) and the proportion of "Breakout" labels (mean $\mu$, standard deviation $\sigma$) based on 100 test episodes for each entropy+ per environment.
}
\begin{tabular}{ccccccccccccccccc}
\toprule
 &  
\multicolumn{2}{c}{Env 1} & \multicolumn{2}{c}{Env 2} & 
\multicolumn{2}{c}{Env 3} & \multicolumn{2}{c}{Env 4} & 
\multicolumn{2}{c}{Env 5} & \multicolumn{2}{c}{Env 6} & 
\multicolumn{2}{c}{Env 7} & \multicolumn{2}{c}{Env 8} \\
\cmidrule(lr){2-17}
entropy+  & $\mu$ & $\sigma$ & $\mu$ & $\sigma$ & $\mu$ & $\sigma$ 
& $\mu$ & $\sigma$ & $\mu$ & $\sigma$ & $\mu$ & $\sigma$ 
& $\mu$ & $\sigma$ & $\mu$ & $\sigma$ \\
\midrule
        0 &        0.81 &       0.06 &        0.74 &       0.16 &        0.35 &       0.07 &        0.56 &       0.06 &        0.47 &       0.10 &        0.49 &       0.06 &        0.62 &       0.12 &        0.57 &       0.10 \\
        1 &        0.80 &       0.05 &        0.76 &       0.09 &        0.33 &       0.07 &        0.56 &       0.03 &        0.46 &       0.10 &        0.44 &       0.11 &        0.53 &       0.18 &        0.53 &       0.11 \\
        2 &        0.77 &       0.09 &        0.70 &       0.13 &        0.32 &       0.08 &        0.54 &       0.07 &        0.47 &       0.11 &        0.43 &       0.13 &        0.53 &       0.19 &        0.52 &       0.13 \\
        3 &        0.77 &       0.07 &        0.72 &       0.12 &        0.31 &       0.07 &        0.53 &       0.11 &        0.45 &       0.12 &        0.41 &       0.10 &        0.52 &       0.20 &        0.46 &       0.15 \\
        4 &        0.73 &       0.10 &        0.69 &       0.10 &        0.31 &       0.06 &        0.48 &       0.12 &        0.44 &       0.12 &        0.40 &       0.11 &        0.52 &       0.18 &        0.46 &       0.12 \\
        5 &        0.71 &       0.10 &        0.67 &       0.12 &        0.31 &       0.06 &        0.48 &       0.12 &        0.36 &       0.17 &        0.38 &       0.11 &        0.46 &       0.17 &        0.46 &       0.15 \\
        6 &        0.70 &       0.09 &        0.68 &       0.09 &        0.29 &       0.07 &        0.48 &       0.13 &        0.36 &       0.15 &        0.36 &       0.11 &        0.46 &       0.16 &        0.45 &       0.15 \\
        7 &        0.71 &       0.10 &        0.65 &       0.11 &        0.27 &       0.07 &        0.45 &       0.12 &        0.38 &       0.15 &        0.37 &       0.11 &        0.47 &       0.19 &        0.41 &       0.15 \\
        8 &        0.68 &       0.11 &        0.62 &       0.12 &        0.30 &       0.05 &        0.46 &       0.11 &        0.40 &       0.14 &        0.33 &       0.09 &        0.41 &       0.18 &        0.41 &       0.16 \\
        9 &        0.66 &       0.12 &        0.63 &       0.11 &        0.28 &       0.06 &        0.42 &       0.13 &        0.35 &       0.16 &        0.32 &       0.10 &        0.39 &       0.19 &        0.40 &       0.14 \\
       10 &        0.64 &       0.14 &        0.63 &       0.12 &        0.27 &       0.07 &        0.40 &       0.14 &        0.36 &       0.14 &        0.34 &       0.09 &        0.34 &       0.15 &        0.38 &       0.15 \\
       11 &        0.69 &       0.11 &        0.62 &       0.12 &        0.27 &       0.07 &        0.42 &       0.14 &        0.33 &       0.15 &        0.28 &       0.12 &        0.36 &       0.17 &        0.42 &       0.15 \\
       12 &        0.62 &       0.13 &        0.60 &       0.10 &        0.26 &       0.06 &        0.41 &       0.14 &        0.35 &       0.15 &        0.29 &       0.09 &        0.36 &       0.14 &        0.32 &       0.12 \\
       13 &        0.64 &       0.13 &        0.61 &       0.11 &        0.28 &       0.06 &        0.39 &       0.15 &        0.31 &       0.15 &        0.26 &       0.10 &        0.32 &       0.15 &        0.35 &       0.16 \\
       14 &        0.62 &       0.14 &        0.59 &       0.13 &        0.26 &       0.08 &        0.37 &       0.16 &        0.33 &       0.16 &        0.29 &       0.09 &        0.30 &       0.14 &        0.38 &       0.11 \\
       15 &        0.60 &       0.12 &        0.58 &       0.10 &        0.27 &       0.05 &        0.36 &       0.13 &        0.34 &       0.13 &        0.27 &       0.11 &        0.31 &       0.14 &        0.30 &       0.14 \\
       16 &        0.61 &       0.15 &        0.57 &       0.14 &        0.25 &       0.07 &        0.35 &       0.12 &        0.32 &       0.14 &        0.28 &       0.10 &        0.28 &       0.13 &        0.32 &       0.15 \\
       17 &        0.59 &       0.14 &        0.58 &       0.10 &        0.25 &       0.06 &        0.37 &       0.12 &        0.31 &       0.13 &        0.25 &       0.09 &        0.28 &       0.13 &        0.32 &       0.15 \\
       18 &        0.57 &       0.15 &        0.56 &       0.12 &        0.26 &       0.07 &        0.36 &       0.14 &        0.27 &       0.16 &        0.28 &       0.09 &        0.28 &       0.15 &        0.28 &       0.13 \\
       19 &        0.56 &       0.15 &        0.56 &       0.12 &        0.22 &       0.07 &        0.35 &       0.13 &        0.28 &       0.14 &        0.27 &       0.09 &        0.30 &       0.12 &        0.30 &       0.13 \\
\bottomrule
\end{tabular}
\end{table}

\begin{table}[!htbp]
\centering
\caption{The relationship between policy entropy increment (entropy+) and the proportion of "Self-improvement" labels (mean $\mu$, standard deviation $\sigma$) based on 100 test episodes for each entropy+ per environment.
}
\begin{tabular}{ccccccccccccccccc}
\toprule
 &  
\multicolumn{2}{c}{Env 1} & \multicolumn{2}{c}{Env 2} & 
\multicolumn{2}{c}{Env 3} & \multicolumn{2}{c}{Env 4} & 
\multicolumn{2}{c}{Env 5} & \multicolumn{2}{c}{Env 6} & 
\multicolumn{2}{c}{Env 7} & \multicolumn{2}{c}{Env 8} \\
\cmidrule(lr){2-17}
entropy+  & $\mu$ & $\sigma$ & $\mu$ & $\sigma$ & $\mu$ & $\sigma$ 
& $\mu$ & $\sigma$ & $\mu$ & $\sigma$ & $\mu$ & $\sigma$ 
& $\mu$ & $\sigma$ & $\mu$ & $\sigma$ \\
\midrule
        0 &        0.09 &       0.03 &        0.10 &       0.04 &        0.22 &       0.03 &        0.20 &       0.02 &        0.20 &       0.09 &        0.36 &       0.05 &        0.10 &       0.05 &        0.18 &       0.09 \\
        1 &        0.09 &       0.03 &        0.12 &       0.04 &        0.23 &       0.05 &        0.19 &       0.02 &        0.21 &       0.08 &        0.36 &       0.06 &        0.16 &       0.12 &        0.21 &       0.12 \\
        2 &        0.10 &       0.04 &        0.12 &       0.05 &        0.24 &       0.06 &        0.20 &       0.06 &        0.22 &       0.09 &        0.38 &       0.07 &        0.17 &       0.13 &        0.21 &       0.13 \\
        3 &        0.10 &       0.04 &        0.12 &       0.05 &        0.23 &       0.07 &        0.21 &       0.07 &        0.23 &       0.13 &        0.37 &       0.05 &        0.17 &       0.10 &        0.26 &       0.12 \\
        4 &        0.11 &       0.04 &        0.13 &       0.05 &        0.23 &       0.04 &        0.22 &       0.06 &        0.22 &       0.10 &        0.38 &       0.07 &        0.18 &       0.11 &        0.24 &       0.12 \\
        5 &        0.12 &       0.04 &        0.14 &       0.05 &        0.23 &       0.06 &        0.22 &       0.05 &        0.30 &       0.17 &        0.37 &       0.07 &        0.18 &       0.10 &        0.22 &       0.11 \\
        6 &        0.11 &       0.04 &        0.14 &       0.05 &        0.22 &       0.07 &        0.22 &       0.08 &        0.25 &       0.13 &        0.37 &       0.07 &        0.18 &       0.10 &        0.25 &       0.12 \\
        7 &        0.13 &       0.05 &        0.13 &       0.05 &        0.25 &       0.08 &        0.23 &       0.07 &        0.26 &       0.13 &        0.39 &       0.07 &        0.17 &       0.11 &        0.28 &       0.15 \\
        8 &        0.12 &       0.04 &        0.14 &       0.06 &        0.24 &       0.07 &        0.20 &       0.04 &        0.24 &       0.12 &        0.37 &       0.07 &        0.20 &       0.13 &        0.28 &       0.12 \\
        9 &        0.12 &       0.04 &        0.15 &       0.06 &        0.23 &       0.06 &        0.23 &       0.07 &        0.24 &       0.12 &        0.37 &       0.07 &        0.19 &       0.12 &        0.29 &       0.13 \\
       10 &        0.13 &       0.06 &        0.14 &       0.06 &        0.22 &       0.07 &        0.23 &       0.08 &        0.24 &       0.16 &        0.38 &       0.06 &        0.21 &       0.10 &        0.32 &       0.10 \\
       11 &        0.12 &       0.05 &        0.14 &       0.06 &        0.24 &       0.07 &        0.22 &       0.07 &        0.24 &       0.10 &        0.38 &       0.09 &        0.22 &       0.12 &        0.25 &       0.10 \\
       12 &        0.13 &       0.05 &        0.14 &       0.06 &        0.23 &       0.08 &        0.23 &       0.07 &        0.23 &       0.12 &        0.37 &       0.08 &        0.24 &       0.15 &        0.31 &       0.12 \\
       13 &        0.12 &       0.04 &        0.14 &       0.05 &        0.23 &       0.06 &        0.22 &       0.06 &        0.25 &       0.13 &        0.37 &       0.08 &        0.23 &       0.14 &        0.28 &       0.12 \\
       14 &        0.13 &       0.06 &        0.14 &       0.06 &        0.22 &       0.08 &        0.21 &       0.06 &        0.24 &       0.15 &        0.36 &       0.08 &        0.22 &       0.12 &        0.28 &       0.11 \\
       15 &        0.13 &       0.05 &        0.14 &       0.06 &        0.24 &       0.08 &        0.24 &       0.08 &        0.23 &       0.12 &        0.35 &       0.08 &        0.21 &       0.10 &        0.31 &       0.13 \\
       16 &        0.14 &       0.05 &        0.14 &       0.05 &        0.21 &       0.07 &        0.22 &       0.06 &        0.24 &       0.14 &        0.37 &       0.06 &        0.20 &       0.10 &        0.29 &       0.12 \\
       17 &        0.13 &       0.04 &        0.16 &       0.06 &        0.23 &       0.07 &        0.21 &       0.06 &        0.23 &       0.11 &        0.37 &       0.08 &        0.27 &       0.15 &        0.29 &       0.13 \\
       18 &        0.15 &       0.06 &        0.15 &       0.06 &        0.22 &       0.07 &        0.22 &       0.07 &        0.26 &       0.15 &        0.35 &       0.07 &        0.26 &       0.12 &        0.33 &       0.15 \\
       19 &        0.13 &       0.05 &        0.17 &       0.06 &        0.21 &       0.07 &        0.22 &       0.07 &        0.23 &       0.12 &        0.36 &       0.07 &        0.20 &       0.11 &        0.27 &       0.13 \\
\bottomrule
\end{tabular}
\end{table}

\begin{table}[!htbp]
\centering
\caption{The relationship between policy entropy increment (entropy+) and the proportion of "Hovering" labels (mean $\mu$, standard deviation $\sigma$) based on 100 test episodes for each entropy+ per environment.
}
\begin{tabular}{ccccccccccccccccc}
\toprule
&  
\multicolumn{2}{c}{Env 1} & \multicolumn{2}{c}{Env 2} & 
\multicolumn{2}{c}{Env 3} & \multicolumn{2}{c}{Env 4} & 
\multicolumn{2}{c}{Env 5} & \multicolumn{2}{c}{Env 6} & 
\multicolumn{2}{c}{Env 7} & \multicolumn{2}{c}{Env 8} \\
\cmidrule(lr){2-17}
entropy+  & $\mu$ & $\sigma$ & $\mu$ & $\sigma$ & $\mu$ & $\sigma$ 
& $\mu$ & $\sigma$ & $\mu$ & $\sigma$ & $\mu$ & $\sigma$ 
& $\mu$ & $\sigma$ & $\mu$ & $\sigma$ \\
\midrule
        0 &        0.04 &       0.04 &        0.10 &       0.18 &        0.28 &       0.07 &        0.07 &       0.07 &        0.15 &       0.07 &        0.09 &       0.06 &        0.12 &       0.08 &        0.09 &       0.05 \\
        1 &        0.05 &       0.04 &        0.06 &       0.07 &        0.29 &       0.08 &        0.09 &       0.05 &        0.15 &       0.05 &        0.13 &       0.09 &        0.17 &       0.13 &        0.12 &       0.07 \\
        2 &        0.07 &       0.07 &        0.11 &       0.13 &        0.29 &       0.10 &        0.10 &       0.06 &        0.15 &       0.08 &        0.12 &       0.08 &        0.18 &       0.14 &        0.13 &       0.08 \\
        3 &        0.06 &       0.05 &        0.10 &       0.12 &        0.33 &       0.10 &        0.11 &       0.08 &        0.15 &       0.06 &        0.16 &       0.09 &        0.20 &       0.14 &        0.17 &       0.10 \\
        4 &        0.09 &       0.08 &        0.10 &       0.09 &        0.33 &       0.08 &        0.16 &       0.10 &        0.17 &       0.08 &        0.16 &       0.09 &        0.19 &       0.13 &        0.17 &       0.09 \\
        5 &        0.10 &       0.07 &        0.12 &       0.10 &        0.32 &       0.08 &        0.15 &       0.10 &        0.20 &       0.12 &        0.19 &       0.11 &        0.25 &       0.15 &        0.20 &       0.11 \\
        6 &        0.10 &       0.06 &        0.11 &       0.09 &        0.36 &       0.13 &        0.16 &       0.10 &        0.24 &       0.18 &        0.19 &       0.09 &        0.23 &       0.14 &        0.19 &       0.11 \\
        7 &        0.09 &       0.07 &        0.14 &       0.10 &        0.37 &       0.10 &        0.18 &       0.11 &        0.22 &       0.15 &        0.19 &       0.09 &        0.23 &       0.14 &        0.20 &       0.10 \\
        8 &        0.12 &       0.10 &        0.16 &       0.13 &        0.33 &       0.09 &        0.19 &       0.12 &        0.19 &       0.10 &        0.24 &       0.10 &        0.27 &       0.15 &        0.21 &       0.13 \\
        9 &        0.14 &       0.10 &        0.14 &       0.10 &        0.35 &       0.11 &        0.22 &       0.13 &        0.26 &       0.18 &        0.25 &       0.10 &        0.30 &       0.14 &        0.20 &       0.11 \\
       10 &        0.15 &       0.14 &        0.15 &       0.12 &        0.38 &       0.12 &        0.22 &       0.13 &        0.25 &       0.16 &        0.23 &       0.09 &        0.32 &       0.14 &        0.20 &       0.13 \\
       11 &        0.11 &       0.07 &        0.15 &       0.12 &        0.36 &       0.13 &        0.22 &       0.14 &        0.29 &       0.18 &        0.27 &       0.11 &        0.29 &       0.16 &        0.22 &       0.12 \\
       12 &        0.17 &       0.13 &        0.17 &       0.10 &        0.41 &       0.10 &        0.23 &       0.13 &        0.27 &       0.18 &        0.27 &       0.10 &        0.28 &       0.14 &        0.26 &       0.14 \\
       13 &        0.16 &       0.11 &        0.17 &       0.11 &        0.36 &       0.10 &        0.25 &       0.15 &        0.28 &       0.17 &        0.31 &       0.12 &        0.32 &       0.18 &        0.26 &       0.13 \\
       14 &        0.16 &       0.11 &        0.18 &       0.13 &        0.40 &       0.12 &        0.28 &       0.18 &        0.29 &       0.18 &        0.28 &       0.09 &        0.36 &       0.17 &        0.24 &       0.12 \\
       15 &        0.19 &       0.12 &        0.18 &       0.11 &        0.38 &       0.09 &        0.26 &       0.13 &        0.26 &       0.12 &        0.31 &       0.12 &        0.35 &       0.16 &        0.29 &       0.15 \\
       16 &        0.18 &       0.15 &        0.20 &       0.14 &        0.43 &       0.12 &        0.29 &       0.13 &        0.30 &       0.20 &        0.28 &       0.10 &        0.38 &       0.14 &        0.29 &       0.16 \\
       17 &        0.19 &       0.13 &        0.17 &       0.10 &        0.41 &       0.11 &        0.26 &       0.14 &        0.31 &       0.19 &        0.30 &       0.09 &        0.33 &       0.15 &        0.29 &       0.18 \\
       18 &        0.20 &       0.15 &        0.20 &       0.13 &        0.40 &       0.11 &        0.29 &       0.16 &        0.34 &       0.21 &        0.29 &       0.10 &        0.34 &       0.16 &        0.29 &       0.12 \\
       19 &        0.23 &       0.14 &        0.19 &       0.13 &        0.45 &       0.12 &        0.28 &       0.14 &        0.35 &       0.18 &        0.29 &       0.10 &        0.37 &       0.15 &        0.33 &       0.19 \\
\bottomrule
\end{tabular}
\end{table}

\begin{table}[!htbp]
\centering
\caption{The relationship between policy entropy increment (entropy+) and the proportion of "Prospect" labels (mean $\mu$, standard deviation $\sigma$) based on 100 test episodes for each entropy+ per environment.
}
\begin{tabular}{ccccccccccccccccc}
\toprule
&  
\multicolumn{2}{c}{Env 1} & \multicolumn{2}{c}{Env 2} & 
\multicolumn{2}{c}{Env 3} & \multicolumn{2}{c}{Env 4} & 
\multicolumn{2}{c}{Env 5} & \multicolumn{2}{c}{Env 6} & 
\multicolumn{2}{c}{Env 7} & \multicolumn{2}{c}{Env 8} \\
\cmidrule(lr){2-17}
entropy+  & $\mu$ & $\sigma$ & $\mu$ & $\sigma$ & $\mu$ & $\sigma$ 
& $\mu$ & $\sigma$ & $\mu$ & $\sigma$ & $\mu$ & $\sigma$ 
& $\mu$ & $\sigma$ & $\mu$ & $\sigma$ \\
\midrule
        0 &        0.06 &       0.03 &        0.06 &       0.02 &        0.15 &       0.03 &        0.17 &       0.02 &        0.18 &       0.04 &        0.06 &       0.02 &        0.16 &       0.07 &        0.16 &       0.05 \\
        1 &        0.07 &       0.03 &        0.06 &       0.02 &        0.15 &       0.04 &        0.17 &       0.03 &        0.18 &       0.05 &        0.07 &       0.03 &        0.14 &       0.07 &        0.13 &       0.06 \\
        2 &        0.07 &       0.03 &        0.07 &       0.03 &        0.14 &       0.04 &        0.16 &       0.04 &        0.16 &       0.05 &        0.07 &       0.02 &        0.12 &       0.06 &        0.14 &       0.05 \\
        3 &        0.07 &       0.03 &        0.06 &       0.03 &        0.13 &       0.05 &        0.15 &       0.04 &        0.17 &       0.06 &        0.07 &       0.04 &        0.11 &       0.05 &        0.12 &       0.05 \\
        4 &        0.07 &       0.04 &        0.07 &       0.03 &        0.14 &       0.04 &        0.15 &       0.04 &        0.17 &       0.05 &        0.06 &       0.03 &        0.12 &       0.05 &        0.13 &       0.05 \\
        5 &        0.07 &       0.04 &        0.08 &       0.03 &        0.14 &       0.05 &        0.15 &       0.03 &        0.15 &       0.07 &        0.06 &       0.04 &        0.12 &       0.06 &        0.12 &       0.05 \\
        6 &        0.08 &       0.05 &        0.07 &       0.03 &        0.12 &       0.05 &        0.14 &       0.03 &        0.15 &       0.08 &        0.08 &       0.03 &        0.12 &       0.06 &        0.11 &       0.05 \\
        7 &        0.08 &       0.04 &        0.08 &       0.03 &        0.12 &       0.05 &        0.15 &       0.04 &        0.15 &       0.06 &        0.06 &       0.03 &        0.12 &       0.06 &        0.11 &       0.05 \\
        8 &        0.07 &       0.04 &        0.08 &       0.03 &        0.13 &       0.04 &        0.15 &       0.04 &        0.17 &       0.05 &        0.06 &       0.03 &        0.13 &       0.05 &        0.10 &       0.05 \\
        9 &        0.08 &       0.04 &        0.08 &       0.03 &        0.13 &       0.04 &        0.14 &       0.04 &        0.15 &       0.07 &        0.07 &       0.04 &        0.13 &       0.05 &        0.10 &       0.05 \\
       10 &        0.08 &       0.04 &        0.08 &       0.03 &        0.13 &       0.04 &        0.14 &       0.05 &        0.15 &       0.07 &        0.06 &       0.04 &        0.13 &       0.05 &        0.09 &       0.04 \\
       11 &        0.07 &       0.04 &        0.08 &       0.04 &        0.12 &       0.04 &        0.14 &       0.04 &        0.13 &       0.07 &        0.08 &       0.04 &        0.12 &       0.05 &        0.11 &       0.04 \\
       12 &        0.08 &       0.05 &        0.09 &       0.04 &        0.11 &       0.04 &        0.13 &       0.04 &        0.15 &       0.06 &        0.07 &       0.04 &        0.12 &       0.05 &        0.10 &       0.05 \\
       13 &        0.08 &       0.04 &        0.08 &       0.04 &        0.12 &       0.05 &        0.14 &       0.04 &        0.16 &       0.09 &        0.06 &       0.04 &        0.13 &       0.05 &        0.10 &       0.05 \\
       14 &        0.09 &       0.04 &        0.08 &       0.04 &        0.12 &       0.05 &        0.14 &       0.04 &        0.14 &       0.08 &        0.07 &       0.04 &        0.13 &       0.04 &        0.10 &       0.05 \\
       15 &        0.09 &       0.04 &        0.09 &       0.04 &        0.11 &       0.05 &        0.14 &       0.04 &        0.17 &       0.08 &        0.07 &       0.04 &        0.13 &       0.06 &        0.10 &       0.05 \\
       16 &        0.08 &       0.05 &        0.10 &       0.04 &        0.12 &       0.05 &        0.14 &       0.04 &        0.14 &       0.08 &        0.08 &       0.04 &        0.14 &       0.06 &        0.09 &       0.06 \\
       17 &        0.09 &       0.04 &        0.09 &       0.04 &        0.12 &       0.05 &        0.15 &       0.04 &        0.15 &       0.07 &        0.08 &       0.04 &        0.12 &       0.06 &        0.10 &       0.05 \\
       18 &        0.08 &       0.04 &        0.08 &       0.04 &        0.12 &       0.04 &        0.14 &       0.05 &        0.13 &       0.08 &        0.08 &       0.04 &        0.12 &       0.05 &        0.10 &       0.05 \\
       19 &        0.09 &       0.04 &        0.08 &       0.04 &        0.12 &       0.04 &        0.14 &       0.04 &        0.15 &       0.07 &        0.08 &       0.04 &        0.13 &       0.05 &        0.10 &       0.06 \\
\bottomrule
\end{tabular}
\end{table}

\newpage
\subsection{Label Convergence and Completion Identification in A2CR Training}
\begin{figure}[h]
  \includegraphics[width=\textwidth]{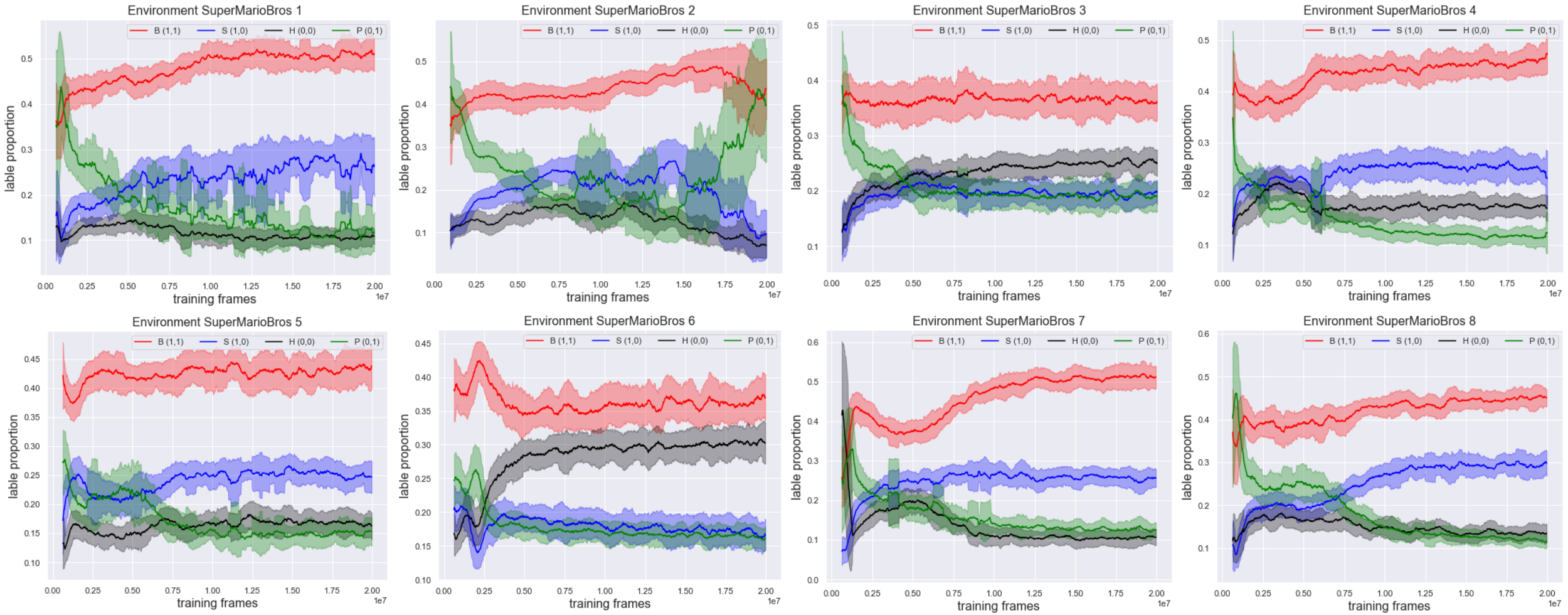}
  \caption{Convergence of four Pseudo-Groundtruth labels in A2CR Training}
  \label{fig:labels_convergence}
\end{figure}
Figure \ref{fig:labels_convergence} illustrates the convergence of the Pseudo-Groundtruth label proportion during A2CR training. In World $\{2\}$, where the agent is not trained to successfully pass the game, the label proportion shows continuous fluctuations without convergence. This indicates that the training is incomplete and requires additional training time. In World $\{8\}$, the agent consistently fails in the last few steps, suggesting that the algorithm is trapped in a local optimal solution. Instead of extending the training time, adjusting training parameters such as $\rho_1$, $\rho_2$, and $\gamma$ becomes necessary at this stage. In other words, the training process is completed but unsuccessful. In World $\{1, 3, 4, 5, 6, 7\}$, all agents are trained to successfully pass the game. The training process is deemed completed in these worlds, supported by the convergence of the label proportions.

\end{document}